\documentclass[letterpaper]{article} 
\usepackage[preprint]{aaai2027}  
\usepackage[hyphens]{url}  
\usepackage{graphicx} 
\usepackage{natbib}  
\usepackage{caption} 
\DeclareCaptionStyle{ruled}{labelfont=normalfont,labelsep=colon,strut=off}
\usepackage{booktabs}
\usepackage{amsmath,amssymb,amsthm}
\usepackage{enumitem}

\newcommand{\MutagenicityN}{4337}
\newcommand{\MutagenicityAvgNodes}{30.32}
\newcommand{\MutagenicityPctSym}{93.4}

\newcommand{\MutagN}{188}
\newcommand{\MutagPctSymModel}{100.0}

\newcommand{\GroupCheckN}{188}

\newcommand{\GroupCheckMaxDiff}{0}
\newcommand{\ProteinsPctSym}{88.4}
\newcommand{\ProteinsN}{250}
\newcommand{\BbbpPctSym}{70.0}
\newcommand{\BbbpN}{250}
\newcommand{\TreeCyclesPct}{38.7}
\newcommand{\TreeGridPct}{44.0}
\newcommand{\SynthJudgments}{1624}

\newcommand{\RealJudgments}{9312}

\newcommand{\RealPredAccAbs}{99.83}
\newcommand{\MutagenicityJudgments}{3594}
\newcommand{\TotalJudgments}{14530}
\newcommand{\ExcN}{4}
\newcommand{\ExcTotal}{884}

\newcommand{\NitroPairMolecules}{25}
\newcommand{\IgWeightRatio}{1.00}
\newcommand{\IgKTenOneOnlyPct}{24.0}
\newcommand{\IgKTenOneOnlyCount}{6}
\newcommand{\IgKTenOneOnlyCiLo}{11.5}
\newcommand{\IgKTenOneOnlyCiHi}{43.4}
\newcommand{\IgKTenArbitraryPct}{100.0}
\newcommand{\SalWeightRatio}{1.00}
\newcommand{\SalKTenOneOnlyPct}{12.0}
\newcommand{\GnnWeightRatioMean}{1.67}
\newcommand{\GnnWeightRatioMax}{2.93}
\newcommand{\GnnKTenOneOnlyPct}{24.0}
\newcommand{\StabDet}{1.0000}
\newcommand{\StabSample}{0.8005}
\newcommand{\StabUnion}{1.0000}
\newcommand{\StabQualIdentical}{93.1}

\newcommand{\GnnOrbitSwitchPct}{0.4}
\newcommand{\CaseN}{25}
\newcommand{\CaseEquivalent}{25}

\newcommand{\CaseTopGraph}{80}

\newcommand{\CaseTopWA}{3.515}
\newcommand{\CaseTopWB}{1.198}
\newcommand{\FixAutMs}{0.11}
\newcommand{\FixAutMsMax}{0.89}
\newcommand{\FixSilentBefore}{83.8}
\newcommand{\FixExtraEdges}{0.43}
\newcommand{\FixFidBefore}{5.1312}
\newcommand{\FixFidAfter}{5.1425}
\newcommand{\UnsafeKMin}{19.5}
\newcommand{\UnsafeKMax}{39.8}
\newcommand{\ReproSymMin}{0.0}
\newcommand{\ReproSymMax}{10.6}
\newcommand{\ReproCtrlMax}{0.0}
\newcommand{\ArchJudgments}{6768}
\newcommand{\ArchPredAcc}{100.00}
\newcommand{\GrandTotalJudgments}{21298}
\newcommand{\AccGCN}{0.867}
\newcommand{\AccGIN}{0.851}
\newcommand{\AccGAT}{0.840}
\newcommand{\EquivAsymOrder}{2e-9}
\newcommand{\OptAsymMin}{0.14}
\newcommand{\OptAsymMax}{0.17}
\newcommand{\ArchEquivSilentMin}{4.8}
\newcommand{\ArchEquivSilentMax}{36.7}
\newcommand{\AccGine}{0.968}

\newcommand{\ManuPctMolecules}{100.0}
\newcommand{\ManuExcessMean}{3.32}
\newcommand{\ManuExcessMax}{16}
\newcommand{\IntrPairsN}{371}
\newcommand{\IntrGcnIndistPct}{100.0}
\newcommand{\IntrGineSepPct}{0.0}
\newcommand{\IntrBlindSepPct}{0.0}
\newcommand{\ManuPairsN}{360}
\newcommand{\ManuGcnIndistPct}{100.0}
\newcommand{\ManuGineSepPct}{99.7}
\newcommand{\ManuBlindSepPct}{0.0}

\newcommand{\ManuEdgePct}{16.2}
\newcommand{\ManuEdgeMax}{90.9}
\newcommand{\AccGineMg}{0.831}
\newcommand{\ManuGineSepPctMg}{13.9}
\newcommand{\ManuGcnIndistPctMg}{100.0}

\newcommand{\ManuExcessMeanMg}{1.74}

\newcommand{\WlCoarserPct}{45.2}
\newcommand{\WlOrbitPairs}{731}
\newcommand{\WlExtraPairs}{325}

\newcommand{\WlNaiveIndistPct}{42.5}
\newcommand{\WlNaiveMax}{$1.1\times 10^{-2}$}
\newcommand{\WlAttrOrbitSalPct}{100.0}
\newcommand{\WlAttrOrbitIgPct}{100.0}
\newcommand{\WlAttrOnlySalPct}{57.9}
\newcommand{\WlAttrOnlyIgPct}{20.0}

\newcommand{\ForcedSevered}{165}
\newcommand{\ForcedAvoidable}{0}
\newcommand{\CrossTieGraphs}{25}

\newtheorem{theorem}{Theorem}
\newtheorem{proposition}{Proposition}
\newtheorem{corollary}{Corollary}
\theoremstyle{definition}

\theoremstyle{remark}
\newtheorem{remark}{Remark}

\usepackage{algorithm}
\usepackage{algorithmic}
\usepackage{tikz}
\usetikzlibrary{arrows.meta,positioning}

\usepackage{fancyvrb}
\usepackage{float}   
\usepackage{cleveref}
\crefname{proposition}{Proposition}{Propositions}
\Crefname{proposition}{Proposition}{Propositions}
\crefname{theorem}{Theorem}{Theorems}
\Crefname{theorem}{Theorem}{Theorems}
\crefname{corollary}{Corollary}{Corollaries}
\Crefname{corollary}{Corollary}{Corollaries}
\crefname{remark}{Remark}{Remarks}
\Crefname{remark}{Remark}{Remarks}
\crefname{section}{Section}{Sections}
\Crefname{section}{Section}{Sections}
\crefname{table}{Table}{Tables}
\Crefname{table}{Table}{Tables}

\newcommand{\Aut}{\mathrm{Aut}}
\newcommand{\Stab}{\mathrm{Stab}}
\newcommand{\Valid}{\mathrm{Valid}}
\newcommand{\WL}{\mathrm{WL}}

\title{Automorphism-Induced Non-Canonicity in\\
Top-$k$ Explanations of Graph Neural Networks}
\author{
    Xin Xu,
    Siru Tao,
    Kaizhen Tan
}
\affiliations{
    Carnegie Mellon University\\
    xuxin@cmu.edu, sirutao@andrew.cmu.edu, kaizhent@cmu.edu
}

\begin{document}
\maketitle

\begin{abstract}
A gradient-based GNN explainer given a molecule with two chemically equivalent
nitro groups assigns them attribution scores that are equal to the last bit. It
cannot do otherwise: message passing is exactly permutation equivariant, so any
automorphism of the input leaves every attribution invariant. Yet the standard
report, the top-$k$ edges, names one of the two, and which one is settled by the
order of an array.

We show this is a structural obstruction rather than an implementation slip.
When no minimal valid explanation is fixed by the input's automorphism group, no
rule can be single-valued, minimal and symmetry-respecting at once. For the
exact-$k$ reports used in practice we give a parameter-free criterion, mechanised
in Lean~4 with no axiom dependencies, that decides from the graph alone whether
\emph{every} score-optimal report of that size must split an orbit. Across
\GrandTotalJudgments{} instance--budget decisions the criterion agrees with a
mechanical model-equivalence check without exception, and no severing case we
found admitted a neutral alternative.

The obstruction is common. Nontrivial automorphisms occur in
\MutagenicityPctSym{}\% of Mutagenicity, the dataset the seminal explainability
papers use, so the measure-zero dismissal of symmetric inputs, sound on the
continuous domains it was made for, collapses here. At the sparsity budget those papers report,
\IgKTenOneOnlyPct{}\% of molecules with two interchangeable nitro groups
(\IgKTenOneOnlyCount{} of \NitroPairMolecules{}) surface exactly one of them,
every one arbitrary under mechanical verification. A model's blindness also manufactures symmetry: every MUTAG
molecule contains atoms chemistry separates and the network provably cannot, and
a matched control shows the resolution is set by what the model reads rather than
how it is parameterised. Reporting orbits removes the arbitrariness at
\FixAutMs{}\,ms and \FixExtraEdges{} extra edges per graph.
\end{abstract}

\section{Introduction}

Consider a molecule with two nitro groups in symmetric positions, and a graph
neural network that predicts mutagenicity. The two groups are interchangeable:
there is an automorphism of the molecular graph carrying one onto the other.
Because message passing is permutation equivariant, the network's output is
invariant under that automorphism, and so is any attribution score computed
from its outputs. A gradient-based explainer therefore scores the two groups
\emph{identically}, and we measure this identity to be exact rather than
approximate: on the \NitroPairMolecules{} molecules of this kind in MUTAG,
Integrated Gradients assigns the two groups a total attribution ratio of
\IgWeightRatio{}$\times$.

The explanation the chemist sees is not the attribution vector. It is a
subgraph, conventionally the top-$k$ edges. And a top-$k$ selection over
values that are equal cannot be neutral: it must order the tied edges
somehow, and common implementations inherit the order of an edge array. At $k=10$, the
budget both \citet{ying2019gnnexplainer} and \citet{luo2020pgexplainer} report
for molecular data, \IgKTenOneOnlyPct{}\% of these molecules have exactly one
of the two equivalent groups surfaced in the reported explanation. In
\IgKTenArbitraryPct{}\% of those cases we verify mechanically that the
alternative---the same explanation carried across by the automorphism---is
accepted by the model to within floating-point tolerance. \Cref{fig:teaser} shows the three stages on
1,4-dinitrobenzene, a molecule of exactly this kind. The chemist is shown one of
two answers the model cannot tell apart, and is not told there was a choice.

\paragraph{What is and is not new here.}
The group theory is not. That an equivariant map cannot break a symmetry of its
input is Curie's principle, stated for machine learning by
\citet{smidt2021finding} and in stabiliser form by
\citet{kaba2023symmetry}. The same monotonicity underlies the impossibility of
continuous canonicalisation \citep{dym2024equivariant} and the inability of
structural representations to separate automorphic nodes
\citep{srinivasan2020equivalence}. Symmetry preservation is an axiom in
\citet{sundararajan2017axiomatic}, whose Theorem~1 makes Integrated Gradients
the unique path method satisfying it, and equivariance of saliency and
Integrated Gradients on graphs is proved by \citet{crabbe2023evaluating}. Both
halves of our starting observation are in print. What is missing is their
composition and its cost. This literature treats equivariance as a property to
achieve and sets symmetric inputs aside as negligible:
\citet{kaba2023symmetry} observe that inputs with a nontrivial stabiliser form a
measure-zero set, true for the continuous domains they consider. On graphs that
argument fails by a wide margin.

\paragraph{Contributions.}
\begin{itemize}[leftmargin=1.2em,itemsep=2pt,topsep=2pt]
\item \textbf{The obstruction.} Single-valuedness, minimality and
symmetry-respect are jointly unachievable exactly when no minimal explanation is
fixed by the input's automorphism group. For the exact-$k$ reports used in
practice we give a parameter-free criterion, mechanised in Lean~4 with no axiom
dependencies, deciding whether every score-optimal report of that size must sever
an orbit.

\item \textbf{The hierarchy, and a prediction it makes.} What the domain
separates, what the model's input separates and what its algorithm separates are
nested. The first gap puts \ManuEdgePct\% of MUTAG's edges beyond the model's
resolution, and a bond-reading network with a matched control shows that reading
an annotation is necessary for an explanation to resolve at the domain's
resolution but not sufficient. The second gap stops at the representation:
refinement classes tie what the forward pass cannot separate, yet attributions
separate most of those pairs, which bounds what an expressiveness result can say
about an explanation.

\item \textbf{The measurement.} Nontrivial automorphisms occur in
\MutagenicityPctSym\% of Mutagenicity ($n{=}\MutagenicityN$), the dataset the
seminal explainability papers actually use, so the measure-zero dismissal does
not transfer from continuous domains to graphs. The criterion agrees with a
mechanical model-equivalence check on all \GrandTotalJudgments{} decisions we
tested. Equivariant explainers split attribution between two interchangeable
functional groups in a ratio of $\IgWeightRatio\times$ and still surface only one
of them \IgKTenOneOnlyPct\% of the time.

\item \textbf{The consequences.} A stability of $1$ is not evidence of
canonicity, since resolving ties by array index attains it while orbit sampling
scores \StabSample{} on outputs of identical fidelity. Reporting orbits removes
the arbitrariness at \FixAutMs{}\,ms and \FixExtraEdges{} extra edges. Separately,
the standard synthetic benchmarks understate the phenomenon because their
generator attaches each motif through a vertex its own automorphism moves.
\end{itemize}

\begin{figure}[t]
\centering
\includegraphics{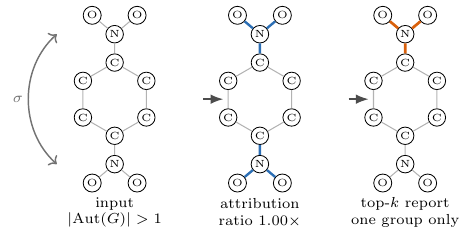}
\caption{1,4-dinitrobenzene, a MUTAG molecule whose two nitro groups are
exchanged by an automorphism $\sigma$. Because message passing is exactly
equivariant, an equivariant explainer must score the two groups equally, and
does: we measure a ratio of $\IgWeightRatio\times$ (centre). A top-$k$ report
must nonetheless name one of them (right), and which one is settled by array
index. At the budget the seminal papers use this happens for
\IgKTenOneOnlyPct\% of such molecules, and in \IgKTenArbitraryPct\% of those the
alternative is accepted by the model to within tolerance.}
\label{fig:teaser}
\end{figure}

\paragraph{Roadmap.}
\Cref{prop:equiv} records that equivariance makes the relevant equalities exact.
\Cref{prop:hierarchy} places that fact in the hierarchy, \Cref{thm:trilemma}
turns it into the impossibility, and \Cref{prop:criterion} decides for a given
instance and budget whether the impossibility is triggered. The results then
measure the incidence, the two gaps, the stage of the pipeline at which the
disparity appears, and the behaviour of the stability metric, before a final
section gives the remedy and its cost. Every number in this paper is produced by
a script from the logged runs, and none is transcribed by hand.

\section{Related Work}
\label{sec:related}

\paragraph{Symmetry of explanations.}
\citet{crabbe2023evaluating} formalise invariance and equivariance of
interpretability methods, prove saliency and Integrated Gradients equivariant,
and estimate violation by Monte Carlo over the full symmetric group. We compute
$\Aut(G)$ exactly with \texttt{nauty} and read the argument in the opposite
direction: equivariance is not a goal to reach but a constraint that forces a
tie. \citet{luong2024rage} observe that PGExplainer's index-order concatenation
makes it not permutation invariant, a different defect from the degeneracy
studied here, and \citet{ling2025equigx} build an explainer for equivariant
networks without establishing equivariance of the explanation itself. Axiomatic
treatments of graph attribution \citep{duval2021graphsvx,bui2024myerson} carry a
symmetry axiom that permutes the value function or asserts equal treatment of
equals, but none requires $\Phi(\pi\cdot G)=\pi\cdot\Phi(G)$. Outside
explainability the same structure is exploited rather than diagnosed:
\citet{bleukx2025mus} show that the image of a minimal unsatisfiable subset under
a symmetry of the instance is again one, and prune enumeration with it. Their
symmetry is our obstruction, put to work. Note that ``invariant explanation'' in
this literature usually means domain invariance for out-of-distribution
generalisation; we use the permutation sense throughout.

\paragraph{Multiplicity, observed but attributed elsewhere.}
\citet{kosan2024gnnxbench} report cross-run Jaccard near $0.5$ and note that
explanations change while their quality does not, which they attribute to
non-convexity. \citet{agarwal2023evaluating} take the maximum over a
hand-annotated set of equally valid ground truths, motivated by semantic
redundancy rather than symmetry, where our criterion generates that set from
$\Aut(G)$ instead of requiring curation. \citet{faber2021when} identify redundant
evidence as a pitfall using an example that is an automorphism.
\citet{azzolin2025reconsidering} record that a top-$k$ operator behaves unstably
given equal scores and randomly permute nodes at load time, which under our
account randomises which orbit representative the user sees.

\section{Theory}
\label{sec:theory}

\subsection{Preliminaries}
\label{sec:setting}

Let $G=(V,E)$ be a finite undirected graph with a node colouring $c$ (the node
features) and let $f$ be a message-passing GNN with frozen weights. We write
$\Aut(G)$ for the group of adjacency- and colour-preserving permutations. An
edge mask $m\in\mathbb{R}^{|E|}$ is fed to $f$ as edge weights, written
$f(G,m)$. An automorphism $\sigma\in\Aut(G)$ acts on masks by permuting
coordinates along the induced edge permutation $\hat\sigma$.

\begin{proposition}[Exact equivariance]
\label{prop:equiv}
For every $\sigma\in\Aut(G)$ and every mask $m$, $f(G,\sigma\cdot m)=f(G,m)$
for graph-level readout, and $f(G,\sigma\cdot m)_{\sigma(v)}=f(G,m)_v$ for
node-level readout.
\end{proposition}

\begin{proof}
By induction on layers. Each layer aggregates over the multiset of neighbour
states, which $\sigma$ preserves because it preserves adjacency and colours.
Graph-level readout is a permutation-invariant pooling, and node-level readout
follows $\sigma$.
\end{proof}

The consequence we use throughout is that equality here is exact, not
approximate. Any validity criterion built from $f$'s outputs---fidelity,
$\varepsilon$-thresholded fidelity, mutual information, a Shapley
value---is therefore \emph{exactly} $\sigma$-invariant, since thresholding two
identical numbers gives identical results. This is stronger than the
corresponding statement for tabular collinearity, where equally good
alternatives must be assumed via a Rashomon set. Here a single frozen model
returns bit-identical outputs, with no assumption beyond the architecture. In
practice on float32 we observe agreement to within one unit in the last place.
Of \ExcTotal{} decisions in our largest single sweep, \ExcN{} differ, all by
exactly $2^{-16}$.

\subsection{A Resolution Hierarchy}
\label{sec:hierarchy}

Which group is the right one is not a bookkeeping choice. Three equivalence
relations on the nodes of $G$ nest, and the gaps between them are the subject of
this paper. Write $u\sim_{\mathrm{dom}}v$ when some automorphism preserving
\emph{all} annotations of the domain carries $u$ to $v$, write $u\sim_f v$ when
one preserving only what $f$ reads does, and write $u\sim_{\WL}v$ when $u$ and
$v$ receive the same colour after $L$ rounds of colour refinement, initialised
with every input a layer consults beyond the multiset of neighbour states.

\begin{proposition}[Resolution hierarchy]
\label{prop:hierarchy}
For an $L$-layer message-passing network $f$,
\[
\sim_{\mathrm{dom}}\ \subseteq\ \sim_f\ \subseteq\ \sim_{\WL},
\]
and $f$ assigns identical representations to any pair related by
$\sim_{\WL}$.
\end{proposition}

\begin{proof}
An annotation-preserving permutation preserves in particular what $f$ reads,
giving the first containment. For the second, an automorphism preserves the
initial colouring and, by induction on rounds, the refined one. The final claim
is the bound of \citet{xu2019powerful} and \citet{morris2019weisfeiler} read at
the node level: a layer is a function of the node's own state and the multiset
of its neighbours', which is what one refinement round records.
\end{proof}

The proposition stops at representations, and deliberately so. For a pair related
by an automorphism, \Cref{prop:equiv} carries equality through to any attribution
because a single permutation acts on the whole computation and on the
perturbation coordinates at once. Two nodes that share a colour without being
automorphic admit no such permutation, and a coordinate-level attribution can
separate them even where the forward pass does not. We measure this below rather
than assume it either way.

\begin{figure}[t]
\centering
\includegraphics{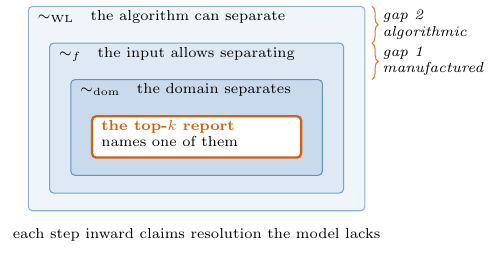}
\caption{The three relations of \Cref{prop:hierarchy}, coarsest outside, with a
top-$k$ report sitting strictly inside all of them. Gap~1 is resolution the model
provably lacks: it is present in \ManuPctMolecules{}\% of MUTAG molecules and
places \ManuEdgePct{}\% of their edges in an orbit the domain splits. Gap~2 is
resolution the forward pass lacks but a coordinate-level attribution can still
supply, which is why only the first gap forces a tie.}
\label{fig:hierarchy}
\end{figure}

Both containments are generally strict, and each strictness has a distinct
meaning (\Cref{fig:hierarchy}). A pair tied under $\sim_f$ but separated under $\sim_{\mathrm{dom}}$ we
call a \emph{manufactured} tie: the model cannot tell the two apart, though the
domain can. Our GCN reads node features and adjacency but not bond types, so a
permutation exchanging the two oxygens of a nitro group---distinguishable in
Kekul\'e form---leaves its output untouched, verified mechanically on all
\GroupCheckN{} MUTAG molecules, where the output moves by
\GroupCheckMaxDiff{} in every case. A pair tied under $\sim_{\WL}$ but not under $\sim_f$ is an \emph{algorithmic}
tie: no automorphism relates the two, yet no network of this depth and class
separates their representations.

The consequence is the thesis of this paper. A top-$k$ report distinguishes
edges, and within a $\sim_f$ class the model does not: there \Cref{prop:equiv}
carries the equality all the way to the attribution, so a report that names one
edge displays a distinction its own subject cannot make. How far the same holds
across the wider $\sim_{\WL}$ classes is an empirical question we answer below.

Let $\Valid(\cdot)$ be any such criterion and
$\mathcal{S}=\{E'\subseteq E:\Valid(E')\}$. By \Cref{prop:equiv},
$\hat\sigma(\mathcal{S})=\mathcal{S}$, so the minimal elements of
$\mathcal{S}$ form a $\hat\sigma$-invariant antichain.

\subsection{The Trilemma}
\label{sec:trilemma}

Call a selection rule $\Phi$, mapping instances to explanation subgraphs,
\emph{neutral} if $\Phi(\sigma\cdot G)=\hat\sigma(\Phi(G))$ for all $\sigma$.

\begin{proposition}[No neutral selection]
\label{prop:neutral}
A neutral $\Phi$ satisfies $\hat\sigma(\Phi(G))=\Phi(G)$ for every
$\sigma\in\Aut(G)$ (for node-level tasks, every $\sigma\in\Stab(v)$). Hence a
neutral rule exists at $G$ if and only if the minimal elements of
$\mathcal{S}$ contain a point fixed by the whole group.
\end{proposition}

The argument is the routine one for breaking a tie between alternatives a rule
cannot distinguish, and we claim no novelty for it. It is the graph instance of
\citet[Prop.~2.1]{kaba2023symmetry}. Its content here is what it forbids.

\begin{theorem}[Conditional trilemma]
\label{thm:trilemma}
Write $H=\Aut(G)$ and let $\operatorname{Min}(\mathcal{S})^{H}$ be the set of
minimal elements of $\mathcal{S}$ fixed by all of $H$. A rule that is
\emph{(i)} single-valued, \emph{(ii)} minimal and \emph{(iii)} neutral exists at
$G$ if and only if $\operatorname{Min}(\mathcal{S})^{H}\neq\varnothing$.
Hence when no minimal valid explanation is $H$-fixed, the three cannot hold
together.
\end{theorem}

\begin{proof}
If $\Phi$ satisfies all three then $\Phi(G)$ is minimal and, by
\Cref{prop:neutral}, $H$-fixed, so $\operatorname{Min}(\mathcal{S})^{H}$ is
nonempty. Conversely, any $E^\ast\in\operatorname{Min}(\mathcal{S})^{H}$ is
itself such a rule's output. The two escapes are then as follows. Relinquishing
\emph{(ii)} and reporting the orbit closure
$\operatorname{cl}_H(E^\ast)=\bigcup_{h\in H}\hat h(E^\ast)$ gives an
$H$-invariant set, and it lies in $\mathcal{S}$ whenever $\mathcal{S}$ is closed
upwards, which is the case for sufficiency-style criteria but not for every
$\Valid$. It contains disjoint sufficient subsets and so is not minimal.
Relinquishing \emph{(i)} and reporting the orbit of $E^\ast$ satisfies
\emph{(ii)} and \emph{(iii)} together.
\end{proof}

Taking $H$ rather than a single $\sigma$ matters: for $|H|>2$ the set
$E^\ast\cup\hat\sigma(E^\ast)$ need not be fixed by the rest of the group.

\begin{corollary}
\label{cor:type}
When $\operatorname{Min}(\mathcal{S})^{H}=\varnothing$, preserving both
minimality and neutrality requires changing the \emph{output type} of an
explainer, from a subgraph to an orbit or to a distribution over minimal
explanations that is equivariant under $H$.
\end{corollary}

\Cref{cor:type} has teeth because every current fidelity, sparsity and
ground-truth metric is defined on a single mask. Its hypothesis is not vacuous: the instances measured below are exactly those where
it holds.

\subsection{When a Report Severs an Orbit}
\label{sec:criterion}

In practice explanations are reported as top-$k$ edges. Suppose the mask is
$\sigma$-invariant, as it must be for an equivariant explainer. Then the edges of
one \emph{edge orbit} all carry the same value and a sort cannot tell them apart.
It can, however, interleave two distinct orbits that happen to share a value,
which is why the two statements below differ.

\begin{proposition}[Combinatorial criterion]
\label{prop:criterion}
Let $E_k$ be the top-$k$ edges under a $\sigma$-invariant mask. Then
$\hat\sigma(E_k)=E_k$ for all $\sigma$ if and only if the cut at $k$ severs no
edge orbit. Otherwise top-$k$ must choose among edges of exactly equal value,
and the alternative $\hat\sigma(E_k)$ is, by \Cref{prop:equiv}, equally valid.
We say the budget \emph{severs} an orbit, reserving \emph{forced} for the
stronger property in \Cref{cor:exactk}.
\end{proposition}

\begin{proof}
If the cut severs no orbit then $E_k$ is a union of complete orbits, and each
$\hat\sigma$ permutes every orbit onto itself, so $\hat\sigma(E_k)=E_k$.
Conversely suppose $\hat\sigma(E_k)=E_k$ for all $\sigma$, and let $O$ be an
orbit with some $e\in O\cap E_k$. For any $e'\in O$ there is, by definition of
an orbit, a $\sigma$ with $\hat\sigma(e)=e'$, and then
$e'\in\hat\sigma(E_k)=E_k$. Hence $O\subseteq E_k$, so every orbit meeting
$E_k$ lies inside it and none is severed.
\end{proof}

\Cref{prop:criterion} characterises a \emph{given} report. Whether \emph{every}
report of that size must sever an orbit is a separate question, and the one faced
when the budget is fixed.

\begin{corollary}[Exact-$k$ reporting]
\label{cor:exactk}
Let $t$ be the $k$-th largest value of a $\sigma$-invariant mask, and split the
edges into $E_{>}=\{e:m_e>t\}$ and $E_{=}=\{e:m_e=t\}$, both unions of complete
orbits. A neutral, score-optimal report of exactly $k$ edges exists if and only if
$k-|E_{>}|$ is a sum of sizes of some sub-collection of the orbits partitioning
$E_{=}$. When it is not, every score-optimal report of that size severs an orbit.
\end{corollary}

\begin{proof}
Any score-optimal report contains $E_{>}$ and takes the remaining
$k-|E_{>}|$ edges from $E_{=}$. By \Cref{prop:criterion} it is neutral exactly
when that remainder is a union of complete orbits, which is possible precisely
for the stated subset sums.
\end{proof}

The distinction is not idle: two orbits may carry the same value, in which case
an index-ordered sort can interleave and sever both where a different
score-optimal choice would have been neutral. On MUTAG \CrossTieGraphs{}
molecules exhibit such cross-orbit ties, yet of the \ForcedSevered{}
instance--budget pairs whose reported top-$k$ severs an orbit,
\ForcedAvoidable{} admit a neutral alternative under \Cref{cor:exactk}. The
severing we observe is therefore not an artefact of the implementation, a
conclusion \Cref{prop:criterion} alone would not support.

The criterion requires no model evaluation: orbits come from the graph, the
ordering from the mask. We mechanise exactly this argument in Lean~4, abstracted
to a statement about a permutation acting on a finite set: $\hat\sigma(S)=S$ if
and only if $S$ is a union of complete orbits. The development is
self-contained, uses no library, and all three theorems discharge with
\emph{no axiom dependencies whatsoever}, not even propositional extensionality.
The criterion therefore holds unconditionally and by machine check, which is what
lets us treat any disagreement with the model below as evidence about the model
rather than about the criterion.

\paragraph{Node-level tasks.}
For a node-level explanation of $v$ the relevant group is $\Stab(v)$, not
$\Aut(G)$: a $\sigma$ moving $v$ relates \emph{two} predictions rather than
exhibiting two explanations of one, which is the equivariance studied by
\citet{crabbe2023evaluating} and not the degeneracy studied here. For an
$L$-layer network only automorphisms acting inside the $L$-hop receptive field
of $v$ can change the explanation. Both restrictions are applied to the
synthetic figures in \Cref{tab:incidence}, and without the second the numbers
would be far larger and largely vacuous: the tree underlying Tree-Grid has an
automorphism group of order exceeding $10^{50}$, almost all of which permutes
undecorated subtrees no explanation can reach.

\begin{remark}[A precision caveat]
\label{rem:subsets}
The orbit action on edges induces one on $j$-subsets, under which two
$j$-subsets of a single orbit need not be related: for $\{a,b,c,d\}$ under
$(a\,b)(c\,d)$, the subsets $\{a,c\}$ and $\{a,d\}$ are related while
$\{a,b\}$ and $\{a,c\}$ are not. \Cref{prop:criterion} therefore detects
correctly \emph{that} a tie is forced, while identifying \emph{which}
alternatives are equivalent needs the subset-level structure. The two coincide
in \StabQualIdentical{}\% of cases, and all validation below uses the direct
model-equivalence check.
\end{remark}

\section{Experimental Setup}
\label{sec:setup}

We explain three-layer GCNs \citep{kipf2017semi} on MUTAG
\citep{debnath1991mutag} and the datasets of \Cref{tab:incidence} with sum pooling, trained to
convergence on each dataset. Explanations of a model that has not learned the
task are meaningless, so we report training accuracy alongside every
measurement and discard any configuration below $0.7$. Automorphism groups are
computed exactly with \texttt{nauty} \citep{mckay2014nauty} through
\texttt{pynauty}, with vertices coloured by their complete feature vector.
Explainers are taken from PyTorch Geometric \citep{fey2019pyg}: GNNExplainer,
PGExplainer, and the Captum family (Saliency, Integrated Gradients,
Input$\times$Gradient, Guided Backpropagation, Deconvolution), seven in all. All
experiments run on CPU in minutes: no GPU is required at any point, and the
whole suite completes on one laptop core (Apple M1 Pro, 16\,GB, macOS 26.5,
PyTorch 2.13, PyTorch Geometric 2.8, Captum 0.9, \texttt{pynauty} 2.8.8; the code
release pins every version).

Hyperparameters were not tuned for performance, only until the model fit the
task: hidden width $64$, Adam at $0.01$, and the epoch count raised from $60$ to
$300$ with dataset size until training accuracy cleared the floor. The
bond-reading network is a three-layer GINE of the same width, trained
identically except on Mutagenicity, where $0.01$ does not converge and we use
$0.003$. Explainer hyperparameters are the library defaults throughout
($100$--$150$ optimisation epochs for GNNExplainer), since the phenomenon is a
property of the reporting step rather than of tuning. Every measurement over
seeds sets \texttt{torch} and \texttt{numpy} seeds before each explainer call and
reports the number of runs: eight per instance for the stability comparison, six
for the cross-seed decomposition, one for the deterministic sweeps.

Two implementation details matter. Each explainer gets a \emph{fresh copy} of the
model from a shared state dictionary, because GNNExplainer installs edge masks on
the message-passing modules and leaves them there, silently disabling the
gradient path the Captum explainers differentiate through. And colouring vertices
by \texttt{argmax} of the feature vector, adequate for the one-hot encodings of
MUTAG and PROTEINS, is wrong for BBBP whose features are mixed integer
descriptors, where it yields spurious automorphisms that do not preserve the
model's input; we colour by the entire vector. Directed edge masks are folded
onto undirected edges by summing the two orientations before any top-$k$.

\section{Results}
\label{sec:results}

\subsection{Incidence of Relevant Symmetry}
\label{sec:incidence}

\begin{table}[t]
\centering
\small
\begin{tabular}{llr}
\toprule
Dataset & Kind & Symmetric \\
\midrule
Mutagenicity & molecules & \MutagenicityPctSym\% \\
MUTAG & molecules & \MutagPctSymModel\% \\
PROTEINS & non-molecular & \ProteinsPctSym\% \\
BBBP & molecules & \BbbpPctSym\% \\
Tree-Grid & synthetic & \TreeGridPct\% \\
Tree-Cycles & synthetic & \TreeCyclesPct\% \\
\bottomrule
\end{tabular}
\caption{Instances admitting a nontrivial automorphism relevant to the explained
prediction, with $n{=}\MutagenicityN$, $\MutagN$, $\ProteinsN$, $\BbbpN$ for the
real datasets \citep{kazius2005mutagenicity,debnath1991mutag,borgwardt2005proteins,wu2018moleculenet},
the first three via TUDataset \citep{morris2020tudataset} and the synthetic
generators from \citet{ying2019gnnexplainer}. Synthetic figures are restricted to
the $3$-hop receptive field of the explained node.}
\label{tab:incidence}
\end{table}

\Cref{tab:incidence} reports incidence, and two points deserve emphasis. First,
Mutagenicity is the dataset the seminal explainability papers actually use:
\citet{ying2019gnnexplainer} describe theirs as ``a dataset of 4,337 molecule
graphs'', and the average of \MutagenicityAvgNodes{} nodes we measure matches
the table in \citet{luo2020pgexplainer}. It is not the 188-graph MUTAG of
TUDataset, the two are frequently conflated, and we report both. Second, the
figures bear on the premise by which symmetric inputs are set aside:
\citet{kaba2023symmetry} dismiss them as measure zero, which holds on the
continuous domains they consider and does not transfer to molecular graphs,
where they are the overwhelming majority. \citet{lawrence2025equivariance} note in
passing that small graphs are especially likely to carry nontrivial
automorphisms, and \Cref{tab:incidence} quantifies it.

\paragraph{Validation of the criterion.}
On Tree-Grid the predicted rate alternates sharply with $k$, falling to zero for
Integrated Gradients at every budget whose cut consumes whole orbits and rising
to the maximum in between, all from the graph alone. We check this against a
mechanical ground truth: form the top-$k$ set, form its image under each automorphism, feed both
to the frozen model as hard masks, and compare outputs. Across \RealJudgments{} decisions on
MUTAG, PROTEINS and BBBP with four explainers, \MutagenicityJudgments{} on
Mutagenicity with three, and \SynthJudgments{} on synthetic benchmarks, the
criterion agrees with the mechanical check on \TotalJudgments{} of
\TotalJudgments{} decisions, with no false positives or negatives. Under an
absolute rather than relative tolerance the figure is \RealPredAccAbs\%, the
discrepant cases differing by one float32 unit in the last place.

\paragraph{An accident in the benchmarks.}
The generator of \citet{ying2019gnnexplainer} attaches each motif through the
motif's \emph{first} node, which for the house motif is one of the two its only
nontrivial automorphism moves, so attaching there destroys the symmetry: we
measure zero symmetric instances in BA-Shapes and in the house half of
BA-2Motifs, and re-attaching through the roof apex restores $|\Aut|=2$. The
community's most-used synthetic benchmarks understate the phenomenon by
construction, and did so by accident.

\subsection{Two Gaps in Resolution}
\label{sec:gaps}

\begin{table}[t]
\centering
\small
\setlength{\tabcolsep}{4pt}
\begin{tabular}{lccc}
\toprule
Node pairs & GCN & GINE-const & GINE \\
\midrule
intrinsic ($n{=}\IntrPairsN$)     & \IntrGcnIndistPct\% & \IntrBlindSepPct\% & \IntrGineSepPct\% \\
manufactured ($n{=}\ManuPairsN$)  & \ManuGcnIndistPct\% & \ManuBlindSepPct\% & \ManuGineSepPct\% \\
\bottomrule
\end{tabular}
\caption{MUTAG. First column: percentage of pairs the GCN cannot separate. Last
two: percentage that network \emph{can}. GINE-const is the same architecture with
bond features replaced by a constant, so it differs from GINE in the information
alone, not the parameterisation. On this dataset GINE separates the manufactured
pairs and ties the intrinsic ones, recovering
$\sim_{\text{GINE}}=\sim_{\mathrm{dom}}$ in both directions. Mutagenicity, below,
behaves differently.}
\label{tab:gaps}
\end{table}

\paragraph{The first gap is pervasive.}
Every MUTAG molecule contains at least one manufactured pair, and the index
$[\sim_f:\sim_{\mathrm{dom}}]$ counting the symmetry the model's blindness
invents averages \ManuExcessMean{} and reaches \ManuExcessMax{}. On average
\ManuEdgePct{}\% of a molecule's edges lie in an orbit the domain splits, rising
to \ManuEdgeMax{}\%. Those edges are related by an automorphism of what the model reads, so
\Cref{prop:equiv} obliges any equivariant explainer to score them equally, and
naming one of them displays a distinction the explainer's own subject cannot
make.

Explaining a network that \emph{does} read bond types tests this, and separates a
necessary claim from a sufficient one. The necessary half is provable and holds
throughout: a bond-blind network cannot separate a manufactured pair, and does
not, on \ManuGcnIndistPct{}\% of MUTAG and \ManuGcnIndistPctMg{}\% of
Mutagenicity pairs. A matched control isolates the cause, since the same GINE
\citep{hu2020strategies} with bond features held constant separates
\ManuBlindSepPct{}\%, exactly like the GCN. What sets the resolution is what the
model reads, not how it is parameterised.

The sufficient half is where the interesting variation lies. On MUTAG a GINE
trained to \AccGine{} separates \ManuGineSepPct{}\% of manufactured pairs and
\IntrGineSepPct{}\% of intrinsic ones (\Cref{tab:gaps}), recovering
$\sim_{\text{GINE}}=\sim_{\mathrm{dom}}$ in both directions. On Mutagenicity the
same construction at \AccGineMg{} separates \ManuGineSepPctMg{}\%, and the
dataset accounts for the difference: it carries three bond types with 82.8\% of
the mass on one against MUTAG's four, and a correspondingly smaller gap, index
\ManuExcessMeanMg{} against \ManuExcessMean{}. Reading an annotation is therefore
necessary for an explanation to resolve at the domain's resolution but not
sufficient, because the task must also reward using it. This is the more useful
conclusion of the two. Over-resolution cannot be repaired by enriching a model's
input, and the resolution at which a given trained model actually operates is an
empirical quantity that the procedure above measures in one pass.

\paragraph{The second gap stops at the representation.}
Automorphism orbits are contained in, and generally strictly smaller than,
colour-refinement classes: on MUTAG the classes are strictly coarser for
\WlCoarserPct{}\% of molecules, adding \WlExtraPairs{} node pairs that no
automorphism relates to the \WlOrbitPairs{} that one does, all of which receive
GCN representations agreeing to within one unit in the last place. It is tempting
to read these as further forced ties. They are not. Saliency assigns equal scores
to \WlAttrOnlySalPct{}\% of them and Integrated Gradients to
\WlAttrOnlyIgPct{}\%, against \WlAttrOrbitSalPct{}\% and
\WlAttrOrbitIgPct{}\% on the automorphic pairs used as a positive control. A
coordinate-level attribution separates what the forward pass cannot, because the
gradient at an edge depends on every path through it and not only on the endpoint
states. Equal representations are therefore necessary for an attribution tie and
not sufficient: an expressiveness bound alone licenses no claim about
explanations, and the counts in \Cref{tab:incidence} should be read as
automorphism-induced ties only.

Matching the refinement to the architecture still matters for the representation
claim. A GCN normalises by $(\tilde d_i\tilde d_j)^{-1/2}$, so a layer consults
its neighbours' degrees and is not a function of the neighbour multiset alone,
and refinement initialised without degree is too coarse to bound it: under that
mismatched calibration only \WlNaiveIndistPct{}\% of the supposedly tied pairs
are tied, with discrepancies up to \WlNaiveMax{}. The signature of a correctly
matched bound is agreement at the last bit.

\subsection{Attribution versus Reporting}
\label{sec:reporting}

\begin{table}[t]
\centering
\small
\setlength{\tabcolsep}{4pt}
\begin{tabular}{llcc}
\toprule
Explainer & Kind & Attr.\ ratio & one group only \\
\midrule
Integrated Grad. & equivariant & $\IgWeightRatio\times$ & \IgKTenOneOnlyPct\% \\
Saliency & equivariant & $\SalWeightRatio\times$ & \SalKTenOneOnlyPct\% \\
GNNExplainer & optimised & $\GnnWeightRatioMean\times$ & \GnnKTenOneOnlyPct\% \\
\bottomrule
\end{tabular}
\caption{The \NitroPairMolecules{} MUTAG molecules whose two nitro groups are
exchanged by an automorphism, at $k=10$. Equivariant explainers split
attribution exactly evenly between the two groups and still surface only one of
them a quarter of the time. Every such case is verified arbitrary: the
automorphic alternative is accepted by the model within tolerance in
\IgKTenArbitraryPct\% of them.}
\label{tab:reporting}
\end{table}

The sharpest form of the problem is visible on the \NitroPairMolecules{} MUTAG
molecules carrying two nitro groups exchanged by an automorphism (\Cref{tab:reporting}). For the
exactly equivariant explainers the total attribution assigned to the two
groups has ratio \IgWeightRatio$\times$: the attribution stage is scrupulously
even-handed. Yet at $k=10$, \IgKTenOneOnlyPct{}\% of these molecules have
exactly one group present in the reported subgraph, and in
\IgKTenArbitraryPct{}\% of those cases the automorphic alternative is accepted
by the model within tolerance. The disparity is manufactured wholly by the
reporting step.

GNNExplainer optimises a mask rather than differentiating one, so it is not
exactly equivariant, and it assigns the two groups ratios up to
\GnnWeightRatioMax$\times$. Its most extreme case is molecule \CaseTopGraph{}.
There two chemically equivalent nitro groups receive \CaseTopWA{} and
\CaseTopWB{}, while the model's outputs on the two corresponding subgraphs are
identical to the last bit. Across all
\CaseN{} such molecules, \CaseEquivalent{} have the two groups verified
model-equivalent.

\paragraph{How often does a practitioner's budget land badly?}
By \Cref{cor:exactk} the admissible budgets are the prefix sums of the orbit
sizes when the orbits carry distinct values, and the subset sums of the tied
block otherwise. Between \UnsafeKMin\% and \UnsafeKMax\% of budgets in $[1,20]$
are of neither form, depending on the dataset. Budgets below the motif size are ordinary practice, not a corner
case: \citet{ying2019gnnexplainer} use $K_M{=}10$ on real-world data,
\citet{luo2020pgexplainer} use ten edges on molecules, \citet{yuan2021subgraphx}
fix a five-node budget for every dataset, and the sweeps in
\citet{amara2022graphframex} and \citet{kosan2024gnnxbench} reach down to five.

\paragraph{Reproducibility.}
Because the tie is settled by array order, permuting the rows of an edge list,
a semantic no-op, can change what is reported. For an equivariant explainer we
observe the reported top-$k$ to change on \ReproSymMin--\ReproSymMax\% of
symmetric graphs and on \ReproCtrlMax\% of asymmetric graphs from the same
datasets and model. PyTorch documents that \texttt{topk} does not guarantee
stable indices among tied elements and may differ between CPU and CUDA, so the
choice is delegated to behaviour the library declines to define, and two widely
used implementations resolve it incompatibly: the original GNNExplainer code
thresholds the mask and keeps both alternatives, while PyTorch Geometric takes a
top-$k$ and keeps one.

\subsection{Stability Metrics under a Severed Orbit}
\label{sec:metric}

Resolve a severed orbit three ways under one $\sigma$-invariant attribution.
Using the cross-run Jaccard of \citet{kosan2024gnnxbench}, breaking the tie by
array index as current practice does scores \StabDet{}, sampling the orbit
\StabSample{}, and reporting the orbit whole \StabUnion{}, while the first two
outputs are automorphic images with identical fidelity in
\StabQualIdentical\% of cases. The metric cannot separate an explainer that is
canonical from one that is deterministic about being arbitrary, since both return
the same set every run and so score $1$ by construction. A stability of $1$
certifies a sorting routine rather than the data, and reporting the orbit
structure beside the score costs one integer per instance.

\begin{remark}[Scope]
\label{rem:notgnn}
This is a statement about what the metric can express, and it applies whether or
not any current explainer is affected. Under an exact test for automorphic
equivalence, \GnnOrbitSwitchPct\% of the GNNExplainer seed pairs whose reported
subgraphs differ do so by an orbit switch, so the instability documented by
\citet{kosan2024gnnxbench} has other sources. The gap in the metric is there
regardless, and it opens exactly when the reporting step meets a severed orbit.
\end{remark}

\subsection{Robustness across Architectures}
\label{sec:arch}

The criterion reads only $\Aut(G)$ and the mask, which \Cref{prop:equiv} makes
orbit-constant for any equivariant explainer and any network. Repeating the
core measurements on MUTAG with a graph isomorphism network and a graph attention
network (training accuracies \AccGCN{}, \AccGIN{}, \AccGAT{}), across
\ArchJudgments{} further instance--budget decisions the
criterion again agrees with the mechanical check on every one
(\ArchPredAcc\%), and the equivariant explainers again surface only one of two
interchangeable groups between \ArchEquivSilentMin\% and
\ArchEquivSilentMax\% of the time. Combined with the earlier sweeps this brings
the validation to \GrandTotalJudgments{} decisions without an exception.

The separation between explainer families is the same on all three
architectures. For the gradient-based explainers the asymmetry of the
attribution under $\sigma$ sits at $\EquivAsymOrder$, which is floating-point
noise, since \Cref{prop:equiv} is exact in exact arithmetic and the residue
depends only on the order in which a sum is accumulated. For GNNExplainer it is
$\OptAsymMin$--$\OptAsymMax$, eight orders of magnitude larger, because a
sparsity-regularised optimiser has no reason to settle on the symmetric point
even when the objective is symmetric about it. GAT is no exception, as
\Cref{prop:hierarchy} anticipates: attention coefficients are normalised over a
neighbourhood, so a layer remains a function of the neighbour multiset.

\section{Orbit-Aware Reporting}
\label{sec:fix}

\begin{algorithm}[t]
\caption{Orbit-aware reporting}
\label{alg:orbit}
\begin{algorithmic}[1]
\STATE {\bfseries Input:} graph $G$, attribution $m$, budget $k$,
       target $v$ (node-level tasks only)
\STATE $H \leftarrow \Aut(G)$ if graph-level, else $\Stab(v)$
       \hfill{\small// \texttt{nauty}, vertices coloured by feature vector}
\STATE $\mathcal{O} \leftarrow$ orbits of $H$ acting on $E$
       \hfill{\small// union--find over generators}
\STATE sort $\mathcal{O}$ by decreasing mean of $m$ over each orbit
\STATE $S \leftarrow \emptyset$
\FOR{$O \in \mathcal{O}$ in sorted order while $|S| < k$}
  \STATE $S \leftarrow S \cup O$ \hfill{\small// a whole orbit, never a part of one}
\ENDFOR
\STATE {\bfseries return} $S$, together with $|\mathcal{O}|$ and $\max_{O}|O|$
       \hfill{\small// so the reader is told a class was reported}
\end{algorithmic}
\end{algorithm}

\Cref{alg:orbit} selects whole orbits in decreasing order of mean attribution
until the budget is met. Its output is $\Aut(G)$-invariant by construction,
satisfying \emph{(i)} and \emph{(iii)} of \Cref{thm:trilemma} and relaxing
strict minimality. On MUTAG it eliminates the arbitrariness entirely, from
\FixSilentBefore\% of budget--instance pairs to none, at a median
\FixAutMs{}\,ms per molecule (\FixAutMsMax{}\,ms worst case) and
\FixExtraEdges{} additional edges displayed. Fidelity here is the drop in the
predicted-class logit when the reported subgraph is deleted,
$f(G)_c-f(G\setminus S)_c$, averaged over instances and larger when the report
carries more of the evidence. It moves from \FixFidBefore{} to
\FixFidAfter{}.

\paragraph{What a practitioner should do.}
Three findings translate directly into practice. First, report $|\Aut(G)|$: it
costs \FixAutMs{}\,ms and tells a reader that symmetry-induced non-canonicity is
possible, the check on whether it is realised being whether the reported set is a
union of complete orbits. Second, prefer a budget landing on an orbit boundary;
between \UnsafeKMin\% and \UnsafeKMax\% of budgets in $[1,20]$ do not, and the
check is combinatorial. Third, if a single subgraph must be shown, show the orbit
containing it and say so. None of these needs a new explainer or retraining.
Benchmark designers should add a fourth: choose a motif's attachment vertex
deliberately, since attaching through one its automorphism group moves removes
the symmetry silently.

\section{Limitations}
\label{sec:limits}

Our models have a few dozen nodes and at most three layers; transformer-style
models, whose relation is coarser still, are untested. We claim no
novelty for the group theory, credited in full above. The validation of
\Cref{prop:criterion} is a consistency check, so given \Cref{prop:equiv} its role
is to establish that nothing outside the criterion is at work rather than to
predict something new. The test of \Cref{prop:hierarchy} against a bond-reading
network carries that burden instead. The nitro-pair case rests on
\NitroPairMolecules{} molecules, so its rates carry wide intervals
(\IgKTenOneOnlyCiLo{}--\IgKTenOneOnlyCiHi\% at $k{=}10$), and at small budgets
most of them have neither group reported at all. Finally, we do not claim any
published figure or qualitative claim is incorrect: our model, training and
dataset version differ from the works we cite, and we measure only the rate at
which the mechanism occurs.

\section{Conclusion}

For graphs, equivariance is not only a property an explanation should have but a
constraint: naming one of two substructures the model cannot separate is a
decision, not a summary.

\bibliography{refs}

\appendix
\setcounter{secnumdepth}{2}
\renewcommand{\thesection}{\Alph{section}}
\onecolumn

\noindent
The appendix records material that does not fit the main text: the machine-checked
development behind the combinatorial criterion, the per-explainer and per-budget
breakdowns behind figures quoted in aggregate, and the experimental detail needed to
reproduce them. The body of the paper is self-contained and nothing here is
load-bearing for any claim it makes. Every number below is generated from the same
archived run outputs as the paper's, by \texttt{scripts/make\_supplement.py}.

\section{The Lean development}

The combinatorial core of the criterion is that a permutation fixes a set setwise
exactly when that set is a union of complete orbits. We mechanise it in Lean~4 over a
finite index type, with no dependency on \texttt{Mathlib} or on any other library. The
development is \texttt{lean/OrbitSplit.lean} in the code archive, is 96 lines long, and
elaborates in under three seconds.

The claim in the paper that the three theorems carry \emph{no axiom dependencies
whatsoever} is checked by the file itself, which ends with three \verb|#print axioms|
commands. Their output, verbatim:

\begin{Verbatim}[fontsize=\small,frame=single,framesep=4pt]
$ lean OrbitSplit.lean
'OrbitSplit.img_eq_iff_closed'
    does not depend on any axioms
'OrbitSplit.not_closed_imp_img_ne'
    does not depend on any axioms
'OrbitSplit.neutrality_excludes_split'
    does not depend on any axioms
\end{Verbatim}

\noindent
Not even propositional extensionality is used. This is a statement about the
combinatorial kernel only. It says the kernel is unconditional and machine-checked; it
does not certify the empirical claims, the choice of automorphism group, or the
implementation of any probe, all of which rest on the experiments instead.
\texttt{img\_eq\_iff\_closed} is the biconditional used as Proposition~4 in the paper;
\texttt{not\_closed\_imp\_img\_ne} is the contrapositive form used when the criterion
predicts a severed orbit; \texttt{neutrality\_excludes\_split} is the statement that a
neutral report cannot cut an orbit.

Reproducing the check requires only a Lean~4 toolchain:

\begin{Verbatim}[fontsize=\small,frame=single,framesep=4pt]
$ lean lean/OrbitSplit.lean
\end{Verbatim}

\section{Extended results}

\Cref{tab:s-breadth} gives every explainer on every dataset of the breadth sweep, rather
than the range quoted in the paper. Two features are worth noting. The gradient-based
family and PGExplainer sit at or below $10^{-8}$ asymmetry, which is accumulation-order
noise, while GNNExplainer is eight orders of magnitude above it. And the rate at which a
report severs an orbit is not monotone in $k$: it rises and falls as the cut crosses
orbit boundaries, which is the signature the criterion predicts from the graph alone.

\Cref{tab:s-nitro} gives the nitro-pair audit at every budget rather than at $k{=}10$
alone. The \emph{arb.} column equals the \emph{one} column in every cell: whenever
exactly one of two interchangeable groups is surfaced, the automorphic alternative is
accepted by the model within tolerance, with no exceptions at any budget for any of the
three explainers. The counts are small, $n{=}\NitroPairMolecules$, and the paper reports
a Wilson interval alongside the headline rate for that reason.

\Cref{tab:s-criterion} gives the criterion against the mechanical check as a confusion
matrix on the synthetic benchmarks. The published BA-2Motifs house generator contributes
no symmetric instances at all, which is the benchmark accident the paper describes; the
repaired attachment point restores them and the criterion is exercised there.

\Cref{tab:s-safek} gives, per dataset, the fraction of budgets in $[1,20]$ that admit a
neutral report. \Cref{tab:s-referee} gives the full distribution behind the paper's
finding that equality of representations does not transfer to equality of attributions.

\begin{table*}[t]\centering\small
\caption{All seven explainers on the three real datasets of the breadth
sweep. \emph{asym} is the mean asymmetry of the attribution under the
automorphism, \emph{agree} the agreement of the criterion with the
mechanical model-equivalence check, and the $k$ columns give the rate at
which the reported top-$k$ severs an orbit. The gradient-based explainers
sit at floating-point noise; GNNExplainer does not.}
\label{tab:s-breadth}
\begin{tabular}{llrrrrrr}
\toprule
Dataset & Explainer & asym & agree & $k{=}2$ & $3$ & $5$ & $8$ \\
\midrule
MUTAG & GNNExplainer & $1.6{\times}10^{-1}$ & 100.0\% & 78.7 & 87.2 & 74.5 & 74.5 \\
 & PGExplainer & $0$ & 100.0\% & 57.4 & 70.2 & 83.0 & 59.6 \\
 & Saliency & $7.0{\times}10^{-9}$ & 100.0\% & 29.8 & 55.3 & 51.1 & 36.2 \\
 & IntegratedGradients & $7.0{\times}10^{-9}$ & 100.0\% & 42.6 & 70.2 & 57.4 & 38.3 \\
 & InputXGradient & $8.0{\times}10^{-9}$ & 100.0\% & 21.3 & 61.7 & 59.6 & 36.2 \\
 & GuidedBackprop & $8.0{\times}10^{-9}$ & 100.0\% & 21.3 & 61.7 & 59.6 & 36.2 \\
 & Deconvolution & $8.0{\times}10^{-9}$ & 100.0\% & 21.3 & 61.7 & 59.6 & 36.2 \\
\midrule
PROTEINS & GNNExplainer & $1.8{\times}10^{-1}$ & 99.4\% & 48.0 & 54.8 & 70.6 & 79.2 \\
 & PGExplainer & $8.2{\times}10^{-28}$ & 99.8\% & 19.9 & 18.6 & 18.6 & 16.3 \\
 & Saliency & $1.6{\times}10^{-9}$ & 99.5\% & 36.2 & 24.9 & 24.0 & 17.2 \\
 & IntegratedGradients & $1.4{\times}10^{-9}$ & 99.7\% & 31.7 & 28.1 & 23.1 & 20.8 \\
 & InputXGradient & $1.6{\times}10^{-9}$ & 99.5\% & 36.2 & 24.9 & 24.0 & 17.2 \\
 & GuidedBackprop & $1.6{\times}10^{-9}$ & 99.5\% & 36.2 & 24.9 & 24.0 & 17.2 \\
 & Deconvolution & $1.6{\times}10^{-9}$ & 99.5\% & 36.2 & 24.9 & 24.0 & 17.2 \\
\bottomrule
\end{tabular}\end{table*}
\begin{table*}[t]\centering\small
\caption{The nitro-pair audit at every budget, on the 25 MUTAG molecules whose two nitro groups are exchanged by an automorphism.
\emph{one} is the count with exactly one group surfaced, \emph{both} with
both, \emph{neither} with neither. \emph{arb.} counts the \emph{one} cases
in which the automorphic alternative is accepted by the model within
tolerance; it equals \emph{one} at every budget and every explainer.}
\label{tab:s-nitro}
\begin{tabular}{lrrrrrrrrrrrrr}
\toprule
 & & \multicolumn{2}{c}{$k{=}3$} & \multicolumn{2}{c}{$k{=}4$} & \multicolumn{2}{c}{$k{=}5$} & \multicolumn{2}{c}{$k{=}6$} & \multicolumn{2}{c}{$k{=}8$} & \multicolumn{2}{c}{$k{=}10$} \\
Explainer & ratio & one & arb. & one & arb. & one & arb. & one & arb. & one & arb. & one & arb. \\
\midrule
GNNExplainer & 1.67$\times$ & 3 & 3 & 3 & 3 & 4 & 4 & 4 & 4 & 4 & 4 & 6 & 6 \\
IntegratedGradients & 1.00$\times$ & 0 & 0 & 0 & 0 & 1 & 1 & 1 & 1 & 6 & 6 & 6 & 6 \\
Saliency & 1.00$\times$ & 0 & 0 & 0 & 0 & 2 & 2 & 2 & 2 & 3 & 3 & 3 & 3 \\
\bottomrule
\end{tabular}\end{table*}
\begin{table*}[t]\centering\small
\caption{Criterion versus mechanical check on the synthetic benchmarks,
as a confusion matrix. TP is a severed orbit that the model confirms
admits an equally valid alternative. The published BA-2Motifs house
generator yields no symmetric instances at all, which is the benchmark
accident discussed in the paper; the repaired attachment restores them.}
\label{tab:s-criterion}
\begin{tabular}{llrrrrr}
\toprule
Explainer & Benchmark & $n$ & TP & FP & FN & TN \\
\midrule
IntegratedGradients & BA-2Motifs cycle5 (published) & 280 & 80 & 0 & 0 & 200 \\
IntegratedGradients & BA-2Motifs house (repaired) & 280 & 80 & 0 & 0 & 200 \\
IntegratedGradients & Tree-Cycles & 112 & 19 & 0 & 0 & 93 \\
IntegratedGradients & Tree-Grid & 140 & 34 & 0 & 0 & 106 \\
Saliency & BA-2Motifs cycle5 (published) & 280 & 80 & 0 & 0 & 200 \\
Saliency & BA-2Motifs house (repaired) & 280 & 80 & 0 & 0 & 200 \\
Saliency & Tree-Cycles & 112 & 20 & 0 & 0 & 92 \\
Saliency & Tree-Grid & 140 & 40 & 0 & 0 & 100 \\
\midrule
\multicolumn{2}{l}{\emph{overall}} & 1624 & \multicolumn{4}{r}{agreement 100.00\%} \\
\bottomrule
\end{tabular}\end{table*}
\begin{table}[H]\centering\small
\caption{Budgets in $[1,20]$ admitting a neutral report, per dataset.
\emph{predicted unsafe} is computed from the orbit sizes alone;
\emph{observed split} is measured from the reported top-$k$. The gap
between them is the quantity the criterion is required to explain.}
\label{tab:s-safek}
\begin{tabular}{lrrrr}
\toprule
Dataset & $n_{\mathrm{sym}}$ & safe & predicted unsafe & observed split \\
\midrule
MUTAG & 47 & 60.2\% & 39.8\% & 43.2\% \\
BBBP & 173 & 77.7\% & 22.3\% & 22.3\% \\
PROTEINS & 221 & 80.5\% & 19.5\% & 19.6\% \\
\bottomrule
\end{tabular}\end{table}
\begin{table*}[t]\centering\small
\caption{Does equality of representations transfer to equality of
attributions? Rows one and two are the positive control, where a genuine
automorphism relates the pair and Proposition 1 forces equality. Rows
three and four are pairs sharing a colour-refinement class without being
automorphic. Equality largely fails there, which is why the paper stops
the hierarchy at the representation.}
\label{tab:s-referee}
\begin{tabular}{llrrrr}
\toprule
Pairs & Attribution & $n$ & equal & median & max \\
\midrule
same orbit & Saliency & 731 & 100.0\% & $0$ & $6.0{\times}10^{-8}$ \\
same orbit & Integrated Grad. & 731 & 100.0\% & $0$ & $1.4{\times}10^{-7}$ \\
same WL class only & Saliency & 325 & 57.9\% & $1.9{\times}10^{-9}$ & $5.3{\times}10^{-1}$ \\
same WL class only & Integrated Grad. & 325 & 20.0\% & $3.6{\times}10^{-3}$ & $1.5{\times}10^{-1}$ \\
\bottomrule
\end{tabular}\end{table*}

\noindent Of the 165 instance--budget pairs on MUTAG whose reported top-$k$ severs an orbit, 0 admit a neutral score-optimal alternative of the same size. 25 molecules do carry two distinct orbits at equal score, so the possibility the corollary describes is realised in the data; it simply never rescues a severed budget here.

\section{The oscillation signature}

\begin{figure}[H]
\centering
\includegraphics{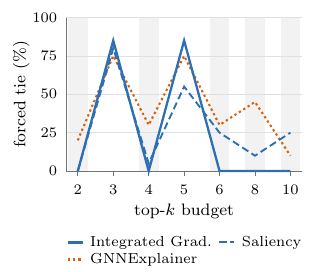}
\caption{The criterion's signature on Tree-Grid, computed from the graph alone before any
model is consulted. Shaded budgets are those whose cut consumes whole edge orbits. For
Integrated Gradients the rate of severing falls to zero at every one of them and rises to
the maximum in between. Saliency alternates in the same phase but less sharply, because
it ranks different orbits into the report. GNNExplainer, whose optimised mask is not
$\sigma$-invariant, never returns to zero at any budget tested. Rates are measured, not
smoothed.}
\label{fig:s-oscillation}
\end{figure}

\Cref{fig:s-oscillation} is the clearest visual form of Proposition~4 and was cut from
the main paper for space. Nothing in the paper depends on it; the same behaviour is
reported numerically in the validation paragraph.

\section{Experimental detail}

\paragraph{Models.}
Three-layer GCNs with hidden width $64$ and sum pooling, trained with Adam at learning
rate $0.01$. The epoch count was raised from $60$ to $300$ with dataset size until
training accuracy cleared the $0.7$ floor, and no further tuning was done: the phenomenon
under study is a property of the reporting step, not of model quality, and a model that
has not learned the task has no explanation worth auditing. The bond-reading network is a
three-layer GINE of the same width trained identically, except on Mutagenicity where
$0.01$ does not converge and $0.003$ is used. GINE-const is the same architecture with
the bond feature replaced by a constant, so that it differs from GINE in the information
available and not in the parameterisation.

\paragraph{Seeds.}
Every measurement over seeds sets \texttt{torch} and \texttt{numpy} seeds explicitly
before each explainer call. The stability comparison uses eight runs per instance, the
cross-seed decomposition six, and the incidence and criterion sweeps are deterministic
given a seed and use one. The main limitation the paper states in this respect is that
model training itself is single-seed, so the GINE separation rates are reported as
measured rather than as means over restarts.

\paragraph{Automorphism groups.}
Computed exactly with \texttt{nauty} through \texttt{pynauty}, with vertices coloured by
their complete feature vector. Colouring by \texttt{argmax} is adequate only for strictly
one-hot encodings and produces spurious automorphisms on mixed integer descriptors such
as BBBP's. Where the bond-preserving reading is wanted, the vertex-coloured group is
computed first and its elements are then filtered to those preserving edge labels, which
yields the correct subgroup rather than an approximation of it.

\paragraph{Tolerance.}
Model equivalence uses the relative test $|o_1-o_2|\le 10^{-4}\max(|o_1|,|o_2|,1)$
throughout. An absolute threshold of $10^{-5}$ misclassifies \ExcN{} of \ExcTotal{}
decisions in the largest sweep, all of which differ by exactly one unit in the last place
of float32 at an output magnitude near $128$.

\paragraph{Cost.}
The full suite runs on CPU on one laptop core; no GPU is used at any point. The
automorphism computation that the remedy requires takes a median of \FixAutMs{}\,ms per
molecule and \FixAutMsMax{}\,ms in the worst case observed.

\end{document}